\documentclass[twocolumn]{ceurart}

\usepackage[ruled,noend,linesnumbered]{algorithm2e}
\usepackage{amsfonts} 
\usepackage{amsthm}
\usepackage{amsmath}
\usepackage{color}
\usepackage{float}
\usepackage{graphicx}
\usepackage{listings}
\usepackage{mathtools}
\usepackage{soul}
\usepackage{subcaption}
\usepackage{tabularx}

\newtheorem{theorem}{Theorem}
\DeclareMathOperator*{\argmin}{argmin}

\DeclareMathOperator*{\MSE}{MSE}
\DeclareMathOperator*{\LL}{LL}
\DeclareMathOperator*{\BF}{BF}
\DeclareMathOperator*{\EE}{\mathbb{E}}

\begin{document}

\copyrightyear{2022}
\copyrightclause{Copyright for this paper by its authors. Use permitted under Creative Commons License Attribution 4.0 International (CC BY 4.0).}

\conference{SafeAI 2023, The AAAI Workshop on Artificial Intelligence Safety, Feb 13--14, 2023, Washington, D.C.}

\title{Active Reward Learning from Multiple Teachers}

\author[1]{Peter Barnett}[email=peterbarnettnz@gmail.com]
\cormark[1]
\author[1]{Rachel Freedman}[]
\author[1]{Justin Svegliato}[]
\author[1]{Stuart Russell}[]

\address[1]{Center for Human-Compatible AI, University of California, Berkeley, CA 94720, USA}

\begin{abstract}
    Reward learning algorithms utilize human feedback to infer a reward function, which is then used to train an AI system. This human feedback is often a preference comparison, in which the human teacher compares several samples of AI behavior and chooses which they believe best accomplishes the objective. While reward learning typically assumes that all feedback comes from a single teacher, in practice these systems often query multiple teachers to gather sufficient training data. In this paper, we investigate this disparity, and find that algorithmic evaluation of these different sources of feedback facilitates more accurate and efficient reward learning. We formally analyze the value of information (VOI) when reward learning from teachers with varying levels of rationality, and define and evaluate an algorithm that utilizes this VOI to actively select teachers to query for feedback. Surprisingly, we find that it is often more informative to query comparatively irrational teachers. By formalizing this problem and deriving an analytical solution, we hope to facilitate improvement in reward learning approaches to aligning AI behavior with human values.
\end{abstract}


\maketitle

\section{Introduction}

Standard AI and machine learning algorithms require the designer to specify a cost or reward function. This objective incentivizes desired behavior and penalizes mistakes, teaching the system how to perform the task. While such objectives are easy to manually specify for problems with clear win conditions, such as games~\cite{silver2016mastering, silver2017mastering,berner2019dota} and tasks with clear goals, such as image classification~\cite{krizhevsky2017imagenet,wang2017residual}, they can be challenging to formalize for more nuanced tasks~\cite{Krakovna2018}. For example, \citet{lee_pebble_2021} find that humans struggle to define an objective that incentivizes bipedal locomotion, despite being experts in both machine learning and walking. By incentivizing incorrect behavior, misspecified objectives can lead to useless or even dangerous outcomes~\cite{Leike2018}. Ensuring that AI systems optimize objectives that align with our own is a crucial part of building safe and beneficial AI.

\textit{Reward learning} techniques enable AI systems to learn their objectives by observing and interacting with humans instead of requiring their designers to specify these objectives manually~\cite{jeon_reward-rational_2020}. Humans can train reward learning systems using a variety of feedback modalities, including demonstrations~\cite{ng2000algorithms,abbeel2004apprenticeship,ziebart2010modeling}, pairwise comparisons~\cite{lee_pebble_2021,sadigh_active_2017,Christiano2017}, natural language~\cite{goyal2019using}, numeric values~\cite{arumugam2019deep}, corrections~\cite{bajcsy2017learning}, and proxy rewards~\cite{hadfield2017inverse,mindermann2018active}. Reward learning from pairwise comparisons in particular has proven remarkably effective across a variety of tasks, including complex physical maneuvers for continuous control systems~\cite{lee_pebble_2021,Christiano2017} and text summarization for language language models~\cite{stiennon2020learning,ziegler2019fine}. In the future, it may even be possible to use reward learning to train AI systems to assist humans in researching safe AI~\cite{Leike2018, leike_schulman_wu_2022}.

However, to infer reward functions from human feedback, reward learning systems must model human decision-making, and incorrect human decision-making models often leads to poor inference~\cite{Skalse2022-np,milli2020literal,freedman2021choice}. Moreover, reward learning systems typically assume that all feedback comes from a single distribution or teacher, despite querying multiple teachers to generate sufficient feedback. However, humans often vary in their expertise, focus, and intelligence, affecting the noisiness of their feedback. The practice of conflating all feedback implicitly disregards the differences between different teachers, increasing the likelihood of human model misspecification and the limitations of reward learning~\cite{daniels2022expertise}. 

In this work, we extend reward learning to take advantage of differences between teachers. We develop a Bayesian reward learning algorithm that actively selects which teacher to query based on the noisiness of their feedback and the learner's current belief. We find that querying a \emph{less} rational teacher can often be more informative than querying a \emph{more} rational teacher, since teacher mistakes inform the agent of the relative values of alternatives. For example, imagine that two teachers are comparing two alternatives, $A$ and $B$. $A$ is worth more than $B$, but only slightly. If the first teacher is perfectly rational, they will always select $A$ over $B$. The learner can infer from this that $A$ is preferable to $B$, but has no way to learn how significant the distinction is. However, assume that the second teacher is somewhat less rational, and occasionally mixes up alternatives of similar value. Then they will typically choose $A$, but sometimes choose $B$, and this allows the learner to infer that the gap between $A$ and $B$ is small. Section 3 formalizes this rationality model and inference procedure.

The rest of the paper is as follows. In Section 2, we discuss prior work on reward learning, active learning, and human modeling. In Section 3, we describe the mechanics of reward learning, including the model of human rationality and the metrics that will be used to measure the value of information (VOI) of teacher feedback. In Section 4, we propose a teacher selection algorithm that selects which teacher to query for feedback at each time step based on the modeled rationality of each teacher and the learner's belief distribution over the reward function. In Sections 5 and 6, we present theoretical and empirical results, showing that the learner's belief will eventually converge to the true reward function under the teacher selection algorithm, that querying less rational teachers can often be more informative, and that our teacher selection method outperforms simple heuristics like always querying the most rational teacher. By formalizing the problem of learning from multiple teachers and deriving an analytical solution, we hope to facilitate improvement in reward learning approaches to value alignment.

\section{Related Work}

\paragraph{Reward Learning}

Reward learning techniques allow AI systems to learn reward functions by observing or interacting with humans. For example, \textit{inverse reinforcement learning} agents observe human behavior or policies, and then infer an underlying reward function that the behavior optimizes~\cite{ng2000algorithms,abbeel2004apprenticeship,ziebart2010modeling}. Recent advances in reward learning have focused on learning from preference comparisons. Here, human teachers observe paired samples of system behavior, then choose which sample they prefer out of each pair. The system learns a reward model that maximizes the likelihood of these preferences, then uses that model to generate a reward signal to guide its behavior. This technique has been successfully applied to many domains, from continuous control~\cite{lee_pebble_2021,Christiano2017} to language generation tasks~\cite{stiennon2020learning,ziegler2019fine}. Reward learning can also use a variety of other feedback modalities, including preference comparisons~\cite{lee_pebble_2021,sadigh_active_2017,Christiano2017}, natural language~\cite{goyal2019using}, numeric values~\cite{arumugam2019deep}, corrections~\cite{bajcsy2017learning}, and proxy rewards~\cite{hadfield2017inverse,mindermann2018active}, but we focus on preference comparisons in this paper due to its recent success.

\paragraph{Active Reward Learning}

Human feedback is expensive and time-consuming to generate, so reward learning algorithms must learn efficiently from limited data. They do this in part by actively selecting the queries that are sent to human teachers in order to maximize the expected VOI of human feedback. ~\citet{sadigh_active_2017} assume that the system is a Bayesian learner, actively synthesizing queries that maximize the expected volume removed from the learner's posterior. ~\citet{biyik_batch_2018} develop efficient approximations to this method and show how to integrate active query selection and reward learning in practice. ~\citet{lee_pebble_2021} take a different approach, empirically evaluating various heuristic strategies for query selection and finding that uncertainty-based sampling methods tend to perform the best. However, all of this previous work focuses on choosing which \textit{queries} to send to the teachers. In this paper, we instead consider which \textit{teachers} to send these queries to.

\paragraph{Human Modeling}

To infer reward functions, AI systems must model the behavior of humans. Early work on reward learning assumed that human behavior was perfectly rational and that human teachers always chose the alternative that maximized their reward~\cite{ng2000algorithms}. Later work models human behavior as pedagogic~\cite{milli2020literal}, systematically biased~\cite{evans_learning_2016}, and noisily or Boltzmann-rational~\cite{jeon_reward-rational_2020,ziebart2010modeling}. We will follow recent work on learning from human preferences~\cite{lee_pebble_2021,jeon_reward-rational_2020,ziebart2010modeling,Christiano2017} and model human teachers as Boltzmann-rational, making choices according to a well-known probability model specified later in the paper.
\begin{figure}[t]
    \centering
    \includegraphics[width=\columnwidth]{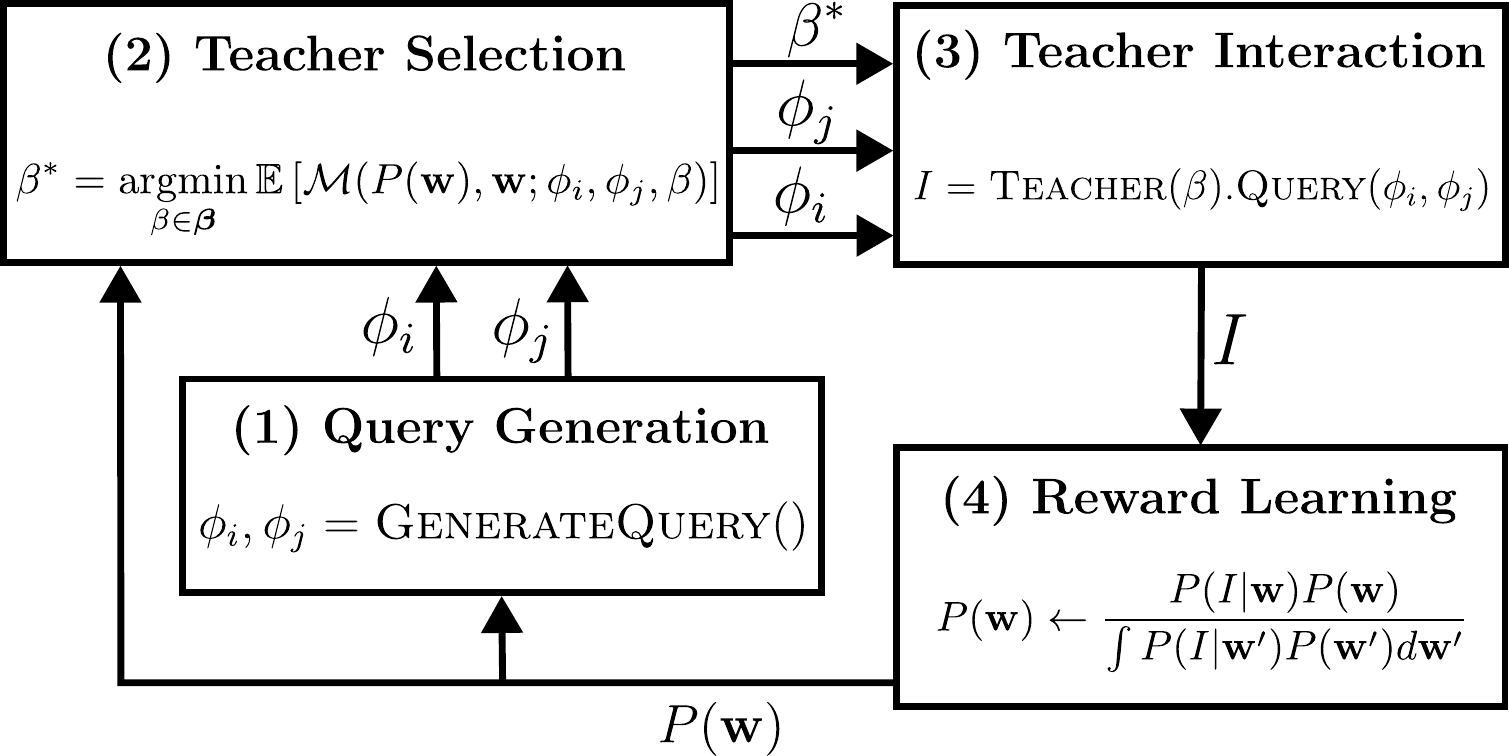}
    \caption{Our active reward learning approach.}
    \label{fig:flow-diagram}
\end{figure}

\section{Active Reward Learning}

In this section, we formalize the problem of selecting the most informative teacher to query in order to gradually learn the correct reward model. In particular, we are interested in greedily selecting the teacher to query at each time step such that the reward model of the agent efficiently converges to the correct reward model. 

At a high level, the teacher selection problem begins with a set of items or trajectories to compare, along with a set of human teachers to evaluate those comparisons. The human teachers each have a different level of rationality that is known \emph{a priori}, meaning that the probability of a given human teacher making a mistake by preferring a less valuable item over a more valuable item is known in advance. During each time step of our approach depicted in Figure~\ref{fig:flow-diagram}, two items are sampled from the set of items (\emph{Step 1}) and then a human teacher is selected to be queried based on these items and the current belief about the reward model (\emph{Step 2}). The human teacher is asked which of the two items they prefer (\emph{Step 3}), and their preference is used to update the reward model (\emph{Step 4}). This process of selecting a query and a teacher is repeated until the reward model converges to the correct reward model.

\emph{Query selection} is the problem of choosing which items to present to the teacher~\cite{lee_pebble_2021}. Some approaches to query selection include choosing the pair of items for which the preference predictors are most uncertain~\cite{lee_pebble_2021,Christiano2017}. Other approaches to query selection include selecting the pair of items that ensure that the space of queries is well covered. Finally, there are more active methods that actively synthesize queries in order learn more efficiently~\cite{sadigh_active_2017, biyik_asking_2020}. Since our focus is on teacher selection rather than query selection, for the purposes of our analysis we will assume that queries are sampled uniformly at random. However, existing methods for query selection can be easily combined with our teacher selection algorithm to further improve reward learning.

To formalize the problem of teacher selection, this section proceeds as follows. We (1)~provide a representation of items and rewards, (2)~apply a well-known model of human rationality to our problem, (3)~offer a method for updating belief distributions that uses preference comparisons from a human teacher, and (4)~propose two metrics that measure the correctness of a belief distribution.

\paragraph{Representing Items and Rewards}

Intuitively, each item can be represented as a set of features. For example, a book could be described by the number of pages and the number of positive reviews or a maneuver made by a self-driving car could be described by its position and distance from other vehicles at each time step. Hence, each item $i$ can formally be represented by a feature vector $\phi_i \in \mathbb{R}^{d}$ where $d$ is the number of features that describe the $i$th item.  

Given this representation of an item, the reward $R(i)$ for an item $i$ can be expressed as a dot product between the feature vector $\phi_i$ and the weight vector $\mathbf{w}\in \mathbb{R}^d$ for the reward model that is being learned: 
\begin{equation}
    R(i) = \mathbf{w}^\top \phi_i. 
\end{equation}
If the items cannot be expressed by a feature vector, this approach can still be used by treating the feature vector $\phi_i$ as a one-hot vector: given the $i$th item, the $i$th entry of the feature vector $\phi_i$ would be $1$ and every other entry would be $0$ while the $i$th entry of the weight vector $\mathbf{w}$ would be the reward $R(i)$ for the $i$th item.

During reward learning, the human teacher is presented with two items and the probability of the human choosing one item over another item depends on the difference in reward between the two items at hand. We therefore express the difference in the reward between two items $i$ and $j$ as the equation 
\begin{equation}
    R(i) - R(j) = \mathbf{w}^\top (\phi_i - \phi_j) = \mathbf{w}^\top\varphi_{ij},
\end{equation}
where $\varphi_{ij} = \phi_i - \phi_j$ is the difference in the feature vectors of the two items.

\paragraph{Modeling Human Rationality}

Human teachers can be represented as Boltzmann-rational agents following a large body of existing work on reward learning~\cite{lee_pebble_2021, jeon_reward-rational_2020, ziebart2010modeling, Christiano2017, bradley_rank_1952, liang_reward_2022, ramachandran_bayesian_2007, palan_learning_2019, freedman2020adapting}. Moreover, we assume that each teacher has a different known rationality parameter $\beta$ rather than assuming $\beta = 1$ for all teachers. Boltzmann-rational teachers are more likely to choose the higher reward item if they are ``more rational" (i.e., a higher $\beta$), or if the difference in reward between the two items is greater. The probability that the teacher chooses an item $i$ over and an item $j$ is given by
\begin{equation} 
    \label{eq:boltzmann-rationality}
    P(i\succ j;\beta) = \frac{\exp(\beta R(i))}{\exp(\beta R(i)) + \exp(\beta R(j))}.
\end{equation}
We thus model the human choice probabilistically: 
\begin{equation}
    \label{eq:P(I|w)}
    P(I|\mathbf{w}; \varphi_{ij}, \beta) = \frac{1}{1 + \exp(-I \beta \mathbf{w}^\top\varphi_{ij})},
\end{equation}
where $I = +1$ if the human prefers item $i$ over item $j$ and $I = -1$ if the human prefers item $j$ over item $i$. This reflects the difference in value of the two items but not their absolute value. Equation \ref{eq:P(I|w)} is a logistic model of the probability of the human preference $I$, where $\beta$ determines the slope. As the difference in reward between the two items increases, the probability that the teacher chooses the higher reward item approaches 1.

\paragraph{Updating Belief Distributions}

The goal of reward learning is to learn the weight vector $\mathbf{w}$ of the reward model. Given the preference of a teacher $I$, the difference in feature vectors $\varphi_{ij}$, and the teacher's rationality parameter $\beta$, the learner updates its belief over the weights of the reward model. That is, the belief over the weights of the reward model is updated such that the reward model now predicts that the item selected by the teacher is more valuable than it was prior to the belief update. Formally, we begin with the current belief distribution $P(\mathbf{w})$, which we treat as the prior distribution, and update it according to Bayes' theorem in the following way:
\begin{equation}
    \label{eq:P(w|I)}
    P(\mathbf{w}|I;\varphi_{ij}, \beta) = \frac{P(I | \mathbf{w}; \varphi_{ij}, \beta) P(\mathbf{w})}{\int P(I|\mathbf{w}'; \varphi_{ij}, \beta) P(\mathbf{w}') d\mathbf{w}'},
\end{equation}
where $P(I|\mathbf{w}; \varphi_{ij}, \beta)$ is given by Equation~\ref{eq:P(I|w)}.

\begin{table*}[t]
    \caption{The general form of an expected metric $\mathcal{M}$ along with the expected metrics for mean squared error (MSE) and log loss (LL).}
    \label{tab:expected-metric-table}

    \small 
    \begin{tabularx}{\linewidth}{Xl} 
        \toprule
        Expected Metric & Equation \\
        \midrule
        $\EE_{\substack{\mathbf{w} \sim P_\mathbf{w} \\ I \sim P_{I|\mathbf{w}}}} \left[ \mathcal{M}(P_{\mathbf{w}|I}, \mathbf{w}; \varphi_{ij}, \beta) \right]$ & $\int P_{\mathbf{w}} \sum_{I} P_{I|\mathbf{w}} \mathcal{M}(P_{\mathbf{w}|I}, \mathbf{w}) d\mathbf{w}$ \\
        \midrule
        $\EE_{\substack{\mathbf{w} \sim P_\mathbf{w} \\ I \sim P_{I|\mathbf{w}}}} \left[ \MSE(P_{\mathbf{w}|I}, \mathbf{w}; \varphi_{ij}, \beta) \right]$ & $2 \sum_{I} \frac{2}{\int f_I(\mathbf{w})d\mathbf{w}} \times \left[ \int f_I(\mathbf{w}) d\mathbf{w} \int f_I(\mathbf{w}) \left\lVert \mathbf{w} \right\rVert^2 d\mathbf{w} - \left\lVert \int f_I(\mathbf{w})\mathbf{w}d\mathbf{w} \right\rVert^2 \right]$ \\
        \midrule 
        $\EE_{\substack{\mathbf{w} \sim P_\mathbf{w} \\ I \sim P_{I|\mathbf{w}}}} \left[ \LL(P_{\mathbf{w}|I}, \mathbf{w}; \varphi_{ij}, \beta) \right]$ & $-\sum_{I} \int f_I(\mathbf{w}) \log \left( \frac{f_I(\mathbf{w})}{\int f_I(\mathbf{w}') d\mathbf{w}'} \right) d\mathbf{w}$ \\
        \bottomrule
    \end{tabularx}
\end{table*}

\paragraph{Measuring Belief Distribution Error}

After querying a teacher and updating the belief over the weights of the reward model $\mathbf{w}$, the belief distribution can be evaluated on a metric that measures the ``correctness'' or the distance of this belief distribution to the true belief distribution. Here, we consider two such metrics: the mean squared error ($\MSE$) and the log loss ($\LL$). The $\MSE$ measure represents how ``far away'' the belief distribution is from the true value while the $\LL$ measure represents the height of the belief distribution at the true value. In both cases, a lower score indicates a more accurate distribution. Using $Q(\mathbf{w})$ as the belief distribution over the weight vector $\mathbf{w}$ and $\mathbf{w}_\emph{\mbox{true}}$ as the true weight vector, the $\MSE$ and $\LL$ measures are given as follows.
\begin{align}
    \label{eq:mean-squared-error}
    \MSE(Q(\mathbf{w}), \mathbf{w}_\emph{\mbox{true}}) &= \int Q(\mathbf{w}) ||\mathbf{w} - \mathbf{w}_\emph{\mbox{true}}||^2 d\mathbf{w} \\
    \label{eq:log-loss}
    \LL(Q(\mathbf{w}), \mathbf{w}_\emph{\mbox{true}}) &= -\log (Q(\mathbf{w}_\emph{\mbox{true}}))
\end{align}
Note that we will describe a greedy approach that selects the teacher that in expectation leads to our belief distribution scoring the best on one of these metrics after a single update in the next section.

Work on active learning from human preferences uses volume removal (i.e., removing as much of the integral of the unnormalized distribution as possible) as a metric~\cite{sadigh_active_2017, biyik_batch_2018, palan_learning_2019}. However, this may not be an appropriate metric for teacher selection. This is because a larger Boltzmann rationality parameter $\beta$ results in a larger volume of the belief distribution being removed but may not necessarily lead to a more accurate belief distribution.

\section{Teacher Selection}

We propose a method for selecting and querying the teacher that produces the best immediate improvement in the expectation of a given metric, which approximates the expected VOI of the teacher feedback. The metrics evaluate how similar the posterior belief is to the ground truth reward, so lower scores indicate improvements in the learned reward model. The algorithm considers uncertainty over two variables: the ground-truth parameterization of the reward model and the item from the query that the teacher prefers. In particular, the expectation of the metric must be taken over the current belief distribution $P(\mathbf{w})$ and the probability $P(I|\mathbf{w};\varphi_{ij},\beta)$ of the teacher preferring each item. Formally, we express the expectation of a given metric $\mathcal{M}$ in Table~\ref{tab:expected-metric-table}. Note that we use the notation $P_\mathbf{w} = P(\mathbf{w})$, $P_{I|\mathbf{w}} = P(I|\mathbf{w};\varphi_{ij}, \beta)$, and $P_{\mathbf{w}|I} = P(\mathbf{w}|I, \varphi_{ij}, \beta)$ throughout this section.

Importantly, the expected value of a given metric only depends on the known variables $\varphi_{ij}$ and $\beta$ along with the current belief distribution $P_\mathbf{w}$ given a straightforward substitution of Equations~\ref{eq:P(I|w)} and~\ref{eq:P(w|I)}. This enables our method to calculate the expected value of the metric for a given teacher with the rationality parameter $\beta$. This will be used to find the teacher to query at each time step: the teacher with the lowest metric in expectation should be selected as that would result in a weight vector that is closest to the true 
weight vector in expectation.

Finally, given the general form of an expected metric, Table~\ref{tab:expected-metric-table} defines the expectations of the MSE and LL metrics using the function $f_I(\mathbf{w}) = P_\mathbf{w} /( 1 + \exp(-I \beta \mathbf{w}^\top \varphi_{ij}))$.

\paragraph{Selecting a Teacher}

To select the teacher to query, we first calculate the expected metric for each teacher $\beta$ given the current belief distribution $P(\mathbf{w})$ and then select the teacher that would result in the lowest expected metric score. Formally, the rationality parameter $\beta^*$ that leads to the largest reduction in the expectation of the metric is defined as follows:
\begin{equation}
    \label{eq:teacher-selection}
     \beta^\ast = \argmin_{\beta\in\boldsymbol{\beta}} \left[ \EE_{\substack{\mathbf{w} \sim P_\mathbf{w} \\ I \sim P_{I|\mathbf{w}}}} \left[ \mathcal{M}(P_{\mathbf{w}|I}, \mathbf{w}; \varphi_{ij}, \beta) \right] \right],
\end{equation}
where $\boldsymbol{\beta}$ is a vector of the $\beta$ values of the teachers.

\paragraph{Learning a Reward Model}

To learn the reward model, the learner begins with an initial belief distribution $P_\mathbf{w}$ over the reward function parameterization and then updates it according to Algorithm \ref{alg:learn-reward-model-algorithm}. First, the algorithm generates queries of paired items and calculates $\beta^\ast$, which is the rationality parameter that leads to the largest improvement in the expectation over the correctness metric. The algorithm queries the teacher with this rationality parameter, and the teacher responds with a \textit{preference} indicating which of the two items in the query they prefer. This preference is used to update the belief distribution $P_\mathbf{w}$. The algorithm iterates until convergence, which is when the entropy of the distribution $P_\mathbf{w}$ becomes lower than a specified threshold $\epsilon$.

\begin{algorithm}[t]
    \small

    \caption{\textsc{LearnRewardModel}$(\cdot)$}
    \label{alg:learn-reward-model-algorithm}

    \DontPrintSemicolon
    \SetKw{Return}{return}
    \SetKw{Not}{not}
    \SetKwInput{Input}{Input}
    \SetKwInput{Output}{Output}

    \Input{An initial belief distribution $P(\mathbf{w})$, a list of the teachers' Boltzmann rationality parameters $\boldsymbol{\beta}$, an expected metric function $\EE[\mathcal{M}]$, and an entropy convergence threshold $\epsilon$}
    \Output{A posterior belief distribution $P(\mathbf{w})$}
    \BlankLine

    $\emph{converged} \gets \emph{False}$ \\
    \While{\Not $\mbox{converged}$} {
        $\phi_i,\ \phi_j \gets \textsc{GenerateQuery}()$ \\
        $\varphi_{ij} \gets \phi_i - \phi_j$ \\
        $\beta^\ast \gets \argmin_{\beta\in\boldsymbol{\beta}} \EE \left[ \mathcal{M}(P(\mathbf{w}),\mathbf{w}; \varphi_{ij}, \beta) \right]$ \\
        \BlankLine
        $I \gets \textsc{Teacher}(\beta^*).\textsc{Query}(\phi_i, \phi_j)$ \\
        $P(\mathbf{w}) \gets \textsc{Normalize}(P(\mathbf{w}) \cdot P(I|\mathbf{w},\varphi_{ij}, \beta^*))$ \\
        \BlankLine
        $\emph{entropy} \gets - \int P(\mathbf{w})\log P(\mathbf{w})d\mathbf{w}$\\
        $\emph{converged} \gets \emph{entropy} < \epsilon$
    }
    \BlankLine
    \Return $P(\mathbf{w})$ \\
\end{algorithm}

\section{Theoretical Analysis}

In this section, we first prove that the belief distribution will converge to the true distribution and then show that, under certain conditions, querying a less rational teacher can result in more informative feedback.

\paragraph{Convergence}

Algorithm~\ref{alg:learn-reward-model-algorithm} queries multiple teachers with different $\beta$ values until the reward estimate convergences. Here, we show that this process will make the belief distribution over $\mathbf{w}$ converge to the true value. 
\begin{theorem}
    \label{thm:convergence}
    In the limit of $N\to\infty$ random queries to Boltzmann-rational teachers with positive, finite $\beta$ values, the posterior distribution over $\mathbf{w}$ converges to the true value.
\end{theorem}

\begin{proof}
    The likelihood of a sequence of human choices $\underline{I}\in[\pm1]^N$ from humans with rationality parameters $\underline{\beta}$ is $P(\underline{I}|\mathbf{w}; \underline{\beta}) = \prod_{i=1}^N P(\underline{I}_i|\mathbf{w}; \underline{\beta}_i)$. The posterior distribution over $\mathbf{w}$ after a sequence of queries is 
    \begin{equation*}
        P(\mathbf{w}|\underline{I}; \underline{\beta}) \propto \prod_i^N P(\underline{I}_i|\mathbf{w}; \underline{\beta}_i) P(\mathbf{w}).
    \end{equation*}
    
    We will show that $P(\mathbf{w}|\underline{I}; \underline{\beta}) \to 0$ as $N\to\infty$ for all $\mathbf{w}\neq \mathbf{w}_\emph{\mbox{true}}$. The Bayes factor between $\mathbf{w}$ and $\mathbf{w}_\emph{\mbox{true}}$ is
    \begin{align*}
        \BF =\frac{P(\mathbf{w}|\underline{I}; \underline{\beta})}{P(\mathbf{w}_\emph{\mbox{true}}|\underline{I}; \underline{\beta})}
         = \frac{\prod_i^N P(\underline{I}_i|\mathbf{w}; \underline{\beta}_i) P(\mathbf{w})}{\prod_i^N P(\underline{I}_i|\mathbf{w}_\emph{\mbox{true}}; \underline{\beta}_i) P(\mathbf{w}_\emph{\mbox{true}})},
    \end{align*}
    where $P(\mathbf{w}_\emph{\mbox{true}}|\underline{I}; \underline{\beta})$ is the posterior distribution at $\mathbf{w}_\emph{\mbox{true}}$. We can show that $\BF\to0$ as $N\to\infty$ except when $\mathbf{w}=\mathbf{w}_\emph{\mbox{true}}$. This implies $P(\mathbf{w}|\underline{I}; \underline{\beta}) \to 0$ except when $\mathbf{w}=\mathbf{w}_\emph{\mbox{true}}$. We require $P(\mathbf{w}_\emph{\mbox{true}})\neq 0$ as $\BF$ is undefined otherwise. Trivially, $\BF=1$ when $\mathbf{w}=\mathbf{w}_\emph{\mbox{true}}$. 

    We now consider $\mathbf{w}\neq\mathbf{w}_\emph{\mbox{true}}$. We can define the negative logarithm of BF, which approaches $\infty$ as $\BF \to 0$:
    \begin{align*}
        -&\log \left(\BF\right) \\ 
        &= -\log\left( \frac{\prod_i^N P(\underline{I}_i|\mathbf{w}; \underline{\beta}_i) P(\mathbf{w})}{\prod_i^N P(\underline{I}_i|\mathbf{w}_\emph{\mbox{true}}; \underline{\beta}_i) P(\mathbf{w}_\emph{\mbox{true}})}\right) \\ 
        &= -\sum_i^N \log \left( \frac{P(\underline{I}_i|\mathbf{w}; \underline{\beta}_i)}{P(\underline{I}_i|\mathbf{w}_\emph{\mbox{true}}; \underline{\beta}_i)}\right) -\log \left(  \frac{P(\mathbf{w})}{ P(\mathbf{w}_\emph{\mbox{true}})}\right).
    \end{align*}

    The first term is the sum of many terms. If this term approaches $\infty$ as $N\to \infty$ then $\BF\to0$. We now examine each term in the sum and show that in expectation they are each positive. All of these terms are independent as they are only depend on the likelihood and not on the current distribution. Hence, they will not decay with additional steps, and so the sum will diverge if the individual terms are positive in expectation. The expected value for each term in the sum is 
    \begin{align*}
        \mathbb{E}&\left[-\log \left( \frac{P(\underline{I}_i|\mathbf{w}; \underline{\beta}_i)}{P(\underline{I}_i|\mathbf{w}_\emph{\mbox{true}}; \underline{\beta}_i)}\right)\right] \\
         &= -\sum_{\mathclap{\underline{I}_i\in+1,-1}} P(\underline{I}_i|\mathbf{w}_\emph{\mbox{true}}; \underline{\beta}_i)\log \left( \frac{P(\underline{I}_i|\mathbf{w}; \underline{\beta}_i)}{P(\underline{I}_i|\mathbf{w}_\emph{\mbox{true}}; \underline{\beta}_i)}\right).
    \end{align*}
    This is the KL divergence between $P(\underline{I}_i|\mathbf{w}_\emph{\mbox{true}}; \underline{\beta}_i)$ and $P(\underline{I}_i|\mathbf{w}; \underline{\beta}_i)$. This is strictly non-negative and only equal to zero when $P(\underline{I}_i|\mathbf{w}; \underline{\beta}_i) = P(\underline{I}_i|\mathbf{w}_\emph{\mbox{true}}; \underline{\beta}_i)$. When $\beta=0$, each of these terms equals $0$. As $\beta \to \infty$, $P(\underline{I}_i|\mathbf{w}; \underline{\beta}_i) \to H(I\mathbf{w}^\top \varphi)$, where $H(\cdot)$ is the Heaviside step function. In this case, it holds that $P(\underline{I}_i|\mathbf{w}; \underline{\beta}_i) = P(\underline{I}_i|\mathbf{w}_\emph{\mbox{true}}; \underline{\beta}_i)$ whenever the values $\mathbf{w}^\top \varphi$ and $\mathbf{w}_\emph{\mbox{true}}^\top \varphi$ have the same sign.

    Therefore, for positive, finite $\beta$ each of the terms in the sum is positive, so the sum diverges, and so the $P(\mathbf{w}|\underline{I}; \underline{\beta})\to0$ for all $\mathbf{w}\neq\mathbf{w}_\emph{\mbox{true}}$. 
\end{proof}

\paragraph{Bigger $\beta$ isn't always more informative}

Querying a more rational teacher (with a larger $\beta$ value) does not always lead to faster convergence to the true value, as measured by lower MSE or LL, because the magnitude of $\mathbf{w}^\top\varphi_{ij}$ can be learned from the teacher making mistakes. 

We empirically observe this in Figure~\ref{fig:optimal-beta-results}, where we demonstrate that if our current belief distribution $P(\mathbf{w})$ is a normal distribution characterized by $\mu$ and $\sigma$, a lower $\beta$ value is more informative for certain values of $\mu$ and $\sigma$. Specifically, when the distribution is symmetric ($\mu = 0$) then a larger value of $\beta$ is better, and as the distribution gets broader (larger $\sigma$) larger $\beta$ is also better. If the distribution is very wide then a large $\beta$ allows us to quickly remove a lot of probability mass, while if the distribution is narrow (and asymmetric) then we learn about the value of $\mathbf{w}^\top\varphi_{ij}$ from the humans making mistakes, which requires the human to be less than perfectly rational. For example, if $\mathbf{w}^\top\varphi_{ij}>0$ then a perfectly rational human would always choose item $i$ over item $j$, and we would not learn about the actual \emph{value} of $\mathbf{w}^\top\varphi_{ij}$.

\begin{figure}[t]
    \centering
    \includegraphics[width=\columnwidth]{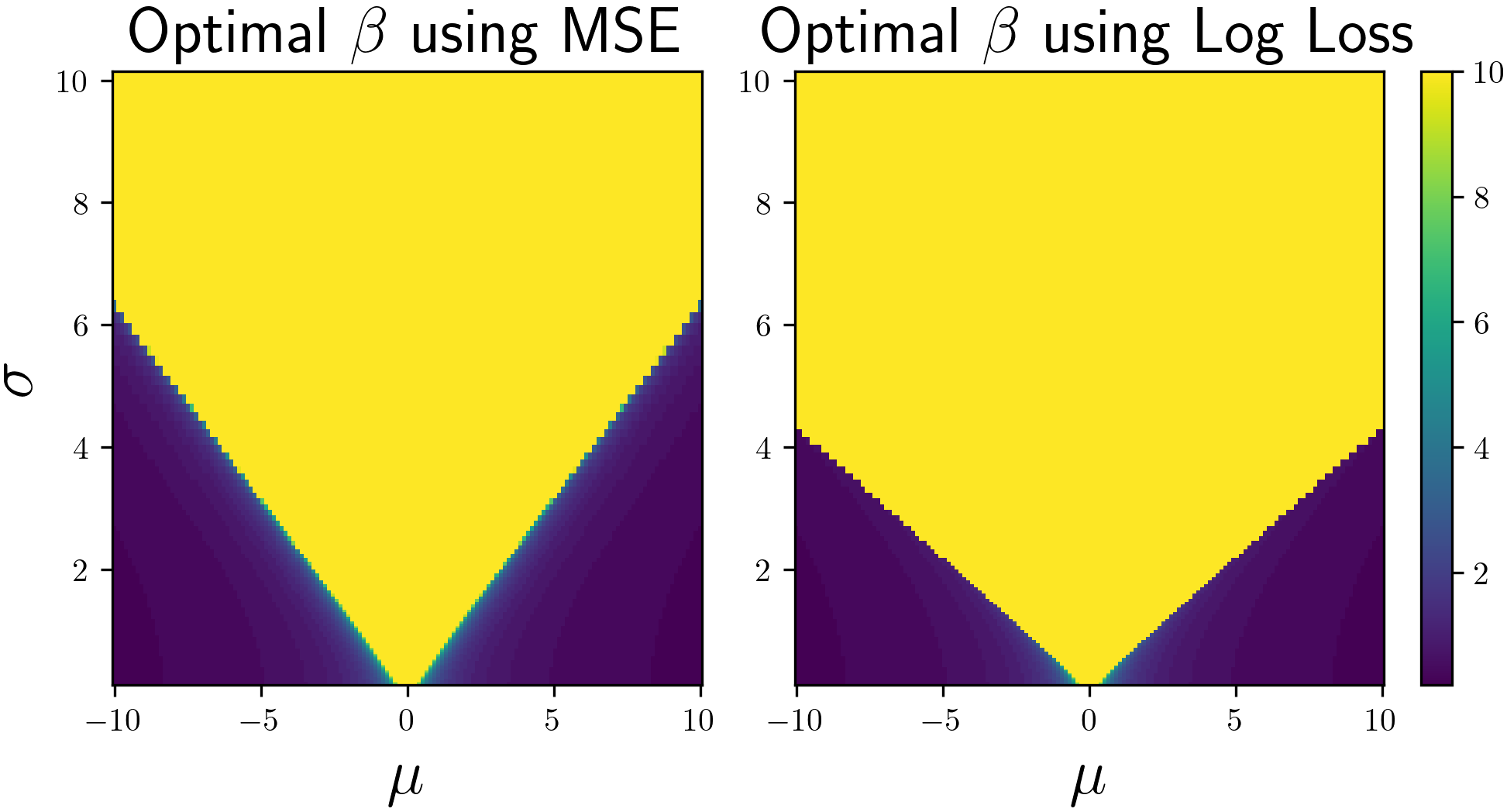}
    \caption{For some prior beliefs over $\mathbf{w}$, querying a teacher with a lower $\beta$ parameter is more informative. The plots show the most informative $\beta$ value (according to the mean squared error and log loss metrics, respectively) for a range of beliefs. Each belief is a Gaussian, parameterized by $\mu$ (horizontal axis) and $\sigma$ (vertical axis). The purple regions of the plots indicate beliefs where it is most informative to query a teacher with a $\beta$ of approximately $1$.}
    \label{fig:optimal-beta-results}
\end{figure}

\section{Restaurant Recommendation}

We now discuss how our method for reward learning using feedback from multiple teachers can be applied to a simplified restaurant recommendation domain. In this domain, the goal is to learn a reward function that can be used to recommend restaurants to a user. This reward model must be learned from feedback from multiple teachers, in this case by asking which of two restaurants a human prefers. It is important to highlight that our approach is compatible with a variety of popular recommendation tasks, including entertainment~\cite{gomez2015netflix,perano2021professional}, news~\cite{raza2021news}, and shopping~\cite{alamdari2020systematic} recommendations.

More formally, the problem of restaurant recommendation has a set of restaurants $\rho = \{\rho_1, \rho_2, \ldots, \rho_n\}$ that can be recommended to a user. Moreover, there is a set of users $U=\{U_1,U_2,\ldots,U_m\}$ who can be queried about their restaurant preferences. Each restaurant is expressed as a set of features $F = \{ \mathrm{Cleanliness}, \mathrm{Vegan}, \mathrm{Spiciness} \}$ where $\mathrm{Cleanliness} \in [1, 10]$ describes the cleanliness of the restaurant, $\mathrm{Vegan} \in \{ 0, 1 \}$ describes whether the restaurant is vegan-friendly, and $\mathrm{Spiciness} \in [1, 10]$ describes the spiciness of the food. The preference rating for each restaurant is denoted by $\mathbf{w}^\top \rho_i$, where $\mathbf{w}\in \mathbb{R}^3$ is a weight vector that parameterizes the reward model. The aim is to learn the weights $\mathbf{w}$ using feedback from multiple users to provide useful restaurant recommendations.

We can represent the restaurant recommendation domain using our approach. The set of items ${\phi_1, \phi_2, \ldots, \phi_n}$ is the set of restaurants $\rho$. The set of human users $U$ is the set of human teachers. The users are modelled as Boltzmann-rational, and have known rationality parameters $\beta_1, \beta_2,\ldots,\beta_m$. Beginning with an initial distribution $P(\mathbf{w})$, we will use Algorithm~\ref{alg:learn-reward-model-algorithm} to converge to the weight values for the reward function that represents the user preferences. First, we select a pair of restaurants for a user to compare (in this case randomly selected) and apply Equation~\ref{eq:teacher-selection} describing which user should be queried in order to achieve the lowest metric score in expectation after a single update. Next, this user is selected and then asked which of the two restaurants they prefer. Finally, using the selected user's preference, the reward model weights are updated according to Equation~\ref{eq:P(w|I)} to generate a new belief distribution. The process is repeated until the belief distribution converges.

\section{Experiments} 

We now show that our approach method for selecting $\beta$ outperforms several baseline methods, using the simple restaurant recommendation domain. In Figure~\ref{fig:metric-convergence-results}, we compare: (1) selecting the largest $\beta$ value to see if the result that larger $\beta$ is not always better is true in practice; (2) selecting $\beta$ randomly to ensure that the advantage over selecting the largest $\beta$ is not just due to the randomness of the selection; and (3) always selecting $\beta = 1$ because this is often what is assumed to be the rationality parameter in other work.

In this experiment, the size of the weight vector is $d = 3$ and the domain of the weights is $W = [-10, 10]^3$, which is discretized. The prior distribution of the weights is a uniform distribution over this domain $P(\mathbf{w})=\mathcal{U}(W)$ and the true weight $\mathbf{w}_\emph{\mbox{true}}\in W$ is sampled from this prior. There are 21 teachers, with $\beta$ values uniformly spaced between 0 and 4. For 100 steps, two restaurant feature vectors $\phi = \{ \mathrm{Cleanliness}, \mathrm{Vegan}, \mathrm{Spiciness} \}$ are generated randomly, where $\mathrm{Cleanliness}, \mathrm{Spiciness}\sim \mathcal{U}(1,10)$, and $\mathrm{Vegan}$ are uniformly drawn from $\{0, 1\}$. While we generate our samples randomly in order to isolate the the effect of teacher selection, any of the active query selection methods from previous work could be used here. The teacher is selected and then queried using one of the various methods and the belief distribution is updated based on the preference of that teacher. The same $\phi$ vectors are used for each method, so that the only difference between the methods is the selection of $\beta$. This procedure is repeated 100 times, each time sampling a new true weight vector $\mathbf{w}_\emph{\mbox{true}}$. 

\begin{figure}[t]
    \centering
    \includegraphics[width=\columnwidth]{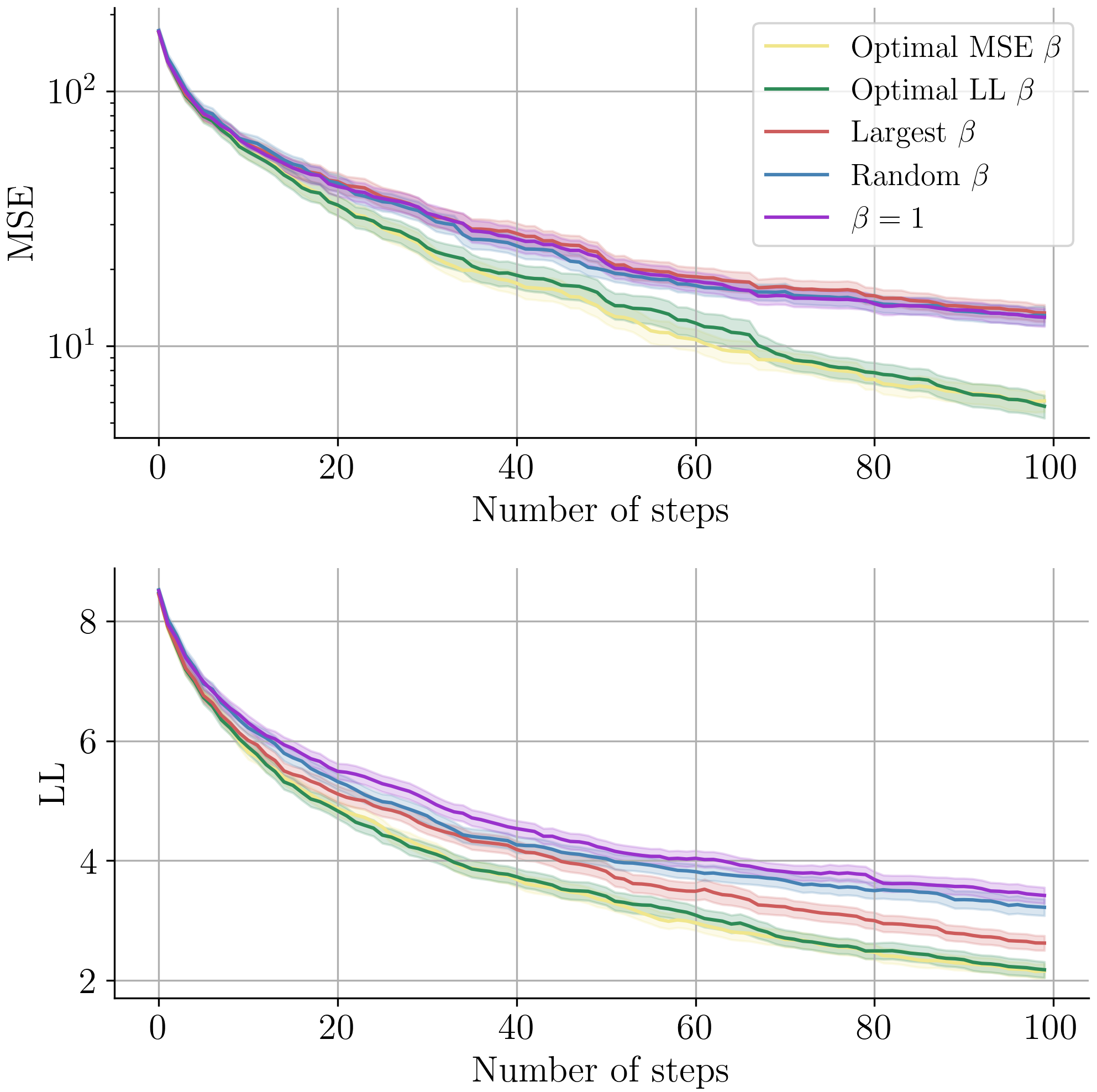}
    \caption{Active teacher selection improves reward inference. These plots show the expected mean squared error and expected log loss over the course of 100 iterations of reward inference using various teacher selection methods. The solid line is the mean, and the shading is the standard deviation. Selecting teacher $\beta$ w.r.t. mean square error most effectively minimizes mean square error, while selecting $\beta$ w.r.t. log loss most effectively minimizes log loss. In both cases, selecting teachers according to Equation~\ref{eq:teacher-selection} clearly outperforms the heuristic of always selecting the most rational teacher (largest $\beta$) and the baselines (random $\beta$ and $\beta=1$).}
    \label{fig:metric-convergence-results}
\end{figure}

Overall, we observe that the active teacher selection methods (MSE and LL) outperform the baseline methods. Moreover, we examine how the most informative value of $\beta$ changes with additional queries in Figure~\ref{fig:beta-convergence-results}. As expected, the optimal $\beta$ value decreases with additional queries, as the distribution gets less broad. At beginning of training, our approach queries the teachers with large $\beta$ values because this enables it to determine the sign of $\mathbf{w}^\top\varphi_{ij}$, and then our approach queries the teachers with smaller $\beta$ values to determine the magnitude of $\mathbf{w}^\top\varphi_{ij}$ as it gets more information.

\begin{figure}[t]
    \centering
    \includegraphics[width=\columnwidth]{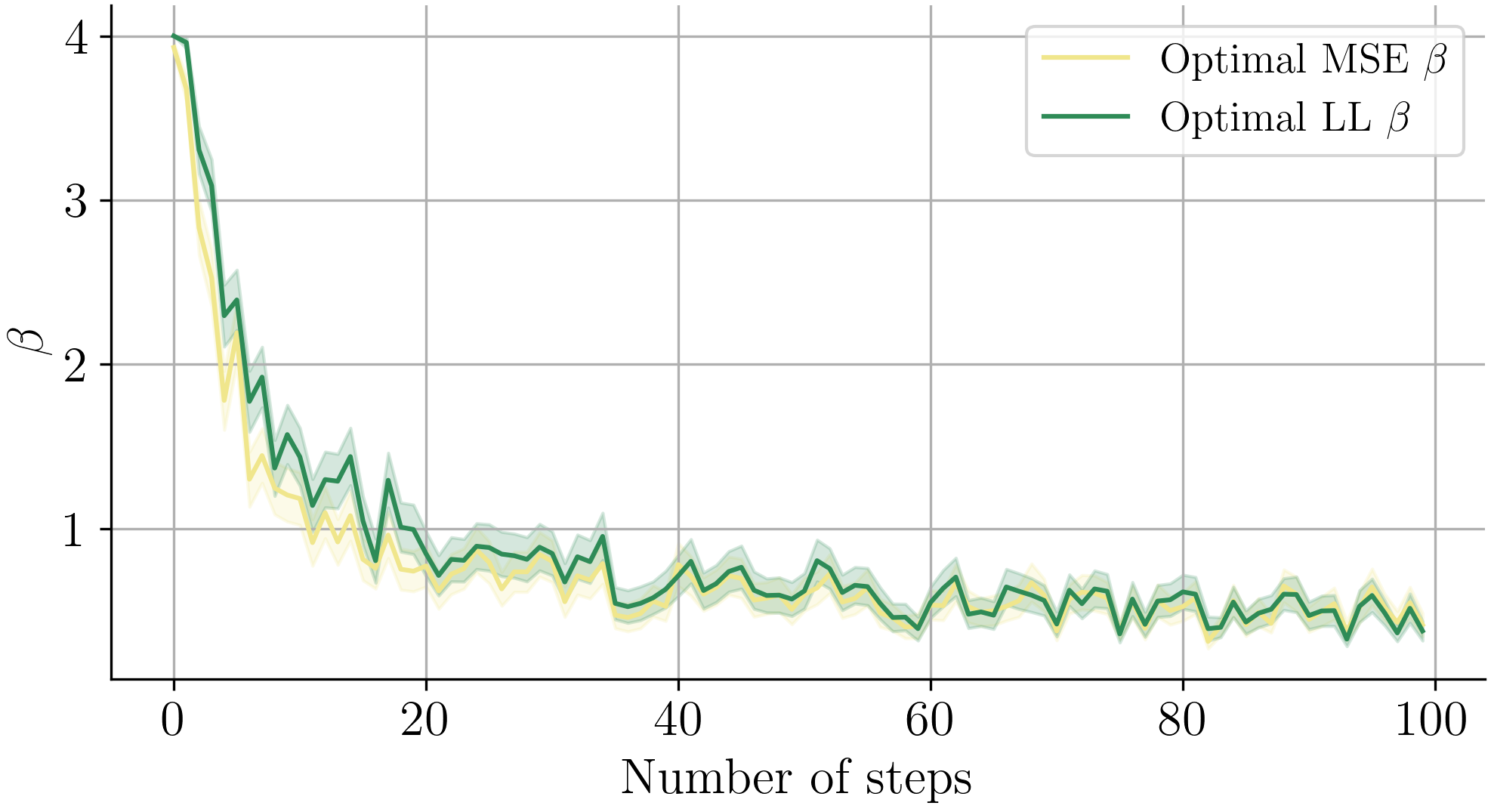}
    \caption{This plot shows the most informative values of $\beta$ during training, averaged across 100 runs (given the expected mean squared error and expected log loss respectively). The solid line is the mean and the shaded area is the standard deviation. $\beta$ decreases over the course of training, as the learner's belief distribution over $\mathbf{w}$ becomes more confident.}
    \label{fig:beta-convergence-results}
\end{figure}

\section{Limitations and Future Work}

For the sake of conceptual clarity and mathematical formalism, we have used relatively simple human decision-making and reward models. Future work should extend these results by increasing model complexity.

For example, this analysis assumes that humans are Boltzmann-rational decision-makers with constant, known $\beta$ values. While more nuanced than optimal models, Boltzmann-rational models fail to account for systematic biases in human judgement~\cite{evans_learning_2016,shah_feasibility_2019,chan_human_2021}. This work could be improved by using more complex, realistic models of human decision-making, for example by allowing each human's $\beta$ parameter to vary across the state space to capture teacher specialization or by measuring and explicitly modeling systematic cognitive biases. Moreover, this analysis assumes that the teacher $\beta$ parameters are given, whereas in reality the agent may not have access to this information. Future work should also examine ways of modeling this part of human decision-making alongside learning the reward function.

Finally, future work could extend these results to nonlinear reward models, such as ensembles of neural networks. Moreover, it could explore convergence properties and optimal querying strategies for learning from teachers with different reward functions. For example, variations in individual taste might lead teachers to disagree on which restaurants are best. Future work should explore the ramifications of such inter-teacher variance on teacher selection and reward learning.

\section{Conclusion}

In this work, we motivated, specified, and evaluated an algorithm for selecting which teacher to query during active reward learning with multiple teachers. Our algorithm models the teachers as Boltzmann-rational with known $\beta$ parameters. At each time step, it queries the teacher that will be most informative in expectation. Interestingly, we find that the most informative teacher is not always the most rational one. We prove and demonstrate that the reward learner's belief will eventually collapse to the true reward function under our algorithm. Our hope is that this method and analysis will improve reward learning in domains where feedback is gathered from multiple teachers with varying levels of rationality.

\section*{Acknowledgments}

We thank the anonymous reviewers for their valuable comments. This work was supported in part by a gift from the Open Philanthropy Foundation.

\bibliography{references}

\begin{thebibliography}{40}
\expandafter\ifx\csname natexlab\endcsname\relax\def\natexlab#1{#1}\fi
\providecommand{\url}[1]{\texttt{#1}}
\providecommand{\href}[2]{#2}
\providecommand{\path}[1]{#1}
\providecommand{\DOIprefix}{doi:}
\providecommand{\ArXivprefix}{arXiv:}
\providecommand{\URLprefix}{URL: }
\providecommand{\Pubmedprefix}{pmid:}
\providecommand{\doi}[1]{\href{http://dx.doi.org/#1}{\path{#1}}}
\providecommand{\Pubmed}[1]{\href{pmid:#1}{\path{#1}}}
\providecommand{\bibinfo}[2]{#2}
\ifx\xfnm\relax \def\xfnm[#1]{\unskip,\space#1}\fi
\bibitem[{Silver et~al.(2016)Silver, Huang, Maddison, Guez, Sifre, Van
  Den~Driessche, Schrittwieser, Antonoglou, Panneershelvam, Lanctot
  et~al.}]{silver2016mastering}
\bibinfo{author}{D.~Silver}, \bibinfo{author}{A.~Huang}, \bibinfo{author}{C.~J.
  Maddison}, \bibinfo{author}{A.~Guez}, \bibinfo{author}{L.~Sifre},
  \bibinfo{author}{G.~Van Den~Driessche}, \bibinfo{author}{J.~Schrittwieser},
  \bibinfo{author}{I.~Antonoglou}, \bibinfo{author}{V.~Panneershelvam},
  \bibinfo{author}{M.~Lanctot}, et~al.,
\newblock \bibinfo{title}{{Mastering the game of Go with deep neural networks
  and tree search}},
\newblock \bibinfo{journal}{Nature} \bibinfo{volume}{529}
  (\bibinfo{year}{2016}) \bibinfo{pages}{484--489}.
\bibitem[{Silver et~al.(2017)Silver, Schrittwieser, Simonyan, Antonoglou,
  Huang, Guez, Hubert, Baker, Lai, Bolton et~al.}]{silver2017mastering}
\bibinfo{author}{D.~Silver}, \bibinfo{author}{J.~Schrittwieser},
  \bibinfo{author}{K.~Simonyan}, \bibinfo{author}{I.~Antonoglou},
  \bibinfo{author}{A.~Huang}, \bibinfo{author}{A.~Guez},
  \bibinfo{author}{T.~Hubert}, \bibinfo{author}{L.~Baker},
  \bibinfo{author}{M.~Lai}, \bibinfo{author}{A.~Bolton}, et~al.,
\newblock \bibinfo{title}{{Mastering the game of Go without human knowledge}},
\newblock \bibinfo{journal}{Nature} \bibinfo{volume}{550}
  (\bibinfo{year}{2017}) \bibinfo{pages}{354--359}.
\bibitem[{Berner et~al.(2019)Berner, Brockman, Chan, Cheung, D{\k{e}}biak,
  Dennison, Farhi, Fischer, Hashme, Hesse et~al.}]{berner2019dota}
\bibinfo{author}{C.~Berner}, \bibinfo{author}{G.~Brockman},
  \bibinfo{author}{B.~Chan}, \bibinfo{author}{V.~Cheung},
  \bibinfo{author}{P.~D{\k{e}}biak}, \bibinfo{author}{C.~Dennison},
  \bibinfo{author}{D.~Farhi}, \bibinfo{author}{Q.~Fischer},
  \bibinfo{author}{S.~Hashme}, \bibinfo{author}{C.~Hesse}, et~al.,
\newblock \bibinfo{title}{{Dota 2 with large scale deep reinforcement
  learning}},
\newblock \bibinfo{journal}{arXiv preprint arXiv:1912.06680}
  (\bibinfo{year}{2019}).
\bibitem[{Krizhevsky et~al.(2017)Krizhevsky, Sutskever, and
  Hinton}]{krizhevsky2017imagenet}
\bibinfo{author}{A.~Krizhevsky}, \bibinfo{author}{I.~Sutskever},
  \bibinfo{author}{G.~E. Hinton},
\newblock \bibinfo{title}{{ImageNet classification with deep convolutional
  neural networks}},
\newblock \bibinfo{journal}{Communications of the ACM} \bibinfo{volume}{60}
  (\bibinfo{year}{2017}) \bibinfo{pages}{84--90}.
\bibitem[{Wang et~al.(2017)Wang, Jiang, Qian, Yang, Li, Zhang, Wang, and
  Tang}]{wang2017residual}
\bibinfo{author}{F.~Wang}, \bibinfo{author}{M.~Jiang},
  \bibinfo{author}{C.~Qian}, \bibinfo{author}{S.~Yang},
  \bibinfo{author}{C.~Li}, \bibinfo{author}{H.~Zhang},
  \bibinfo{author}{X.~Wang}, \bibinfo{author}{X.~Tang},
\newblock \bibinfo{title}{{Residual attention network for image
  classification}},
\newblock in: \bibinfo{booktitle}{IEEE Conference on Computer Vision and
  Pattern Recognition}, \bibinfo{year}{2017}, pp. \bibinfo{pages}{3156--3164}.
\bibitem[{Krakovna(2018)}]{Krakovna2018}
\bibinfo{author}{V.~Krakovna}, \bibinfo{title}{{Specification gaming examples
  in AI}}, \bibinfo{year}{2018}.
\bibitem[{Lee et~al.(2021)Lee, Smith, and Abbeel}]{lee_pebble_2021}
\bibinfo{author}{K.~Lee}, \bibinfo{author}{L.~M. Smith},
  \bibinfo{author}{P.~Abbeel},
\newblock \bibinfo{title}{{PEBBLE: Feedback-efficient interactive reinforcement
  learning via relabeling experience and unsupervised pre-training}},
\newblock in: \bibinfo{booktitle}{38th International Conference on Machine
  Learning}, \bibinfo{publisher}{PMLR}, \bibinfo{year}{2021}, pp.
  \bibinfo{pages}{6152--6163}.
\bibitem[{Leike et~al.(2018)Leike, Krueger, Everitt, Martic, Maini, and
  Legg}]{Leike2018}
\bibinfo{author}{J.~Leike}, \bibinfo{author}{D.~Krueger},
  \bibinfo{author}{T.~Everitt}, \bibinfo{author}{M.~Martic},
  \bibinfo{author}{V.~Maini}, \bibinfo{author}{S.~Legg},
\newblock \bibinfo{title}{{Scalable agent alignment via reward modeling: A
  research direction}},
\newblock \bibinfo{journal}{arXiv preprint arXiv:1811.07871}
  (\bibinfo{year}{2018}).
\bibitem[{Jeon et~al.(2020)Jeon, Milli, and Dragan}]{jeon_reward-rational_2020}
\bibinfo{author}{H.~J. Jeon}, \bibinfo{author}{S.~Milli},
  \bibinfo{author}{A.~D. Dragan},
\newblock \bibinfo{title}{{Reward-rational (implicit) choice: A unifying
  formalism for reward learning}},
\newblock \bibinfo{journal}{arXiv preprint arXiv:2002.04833}
  (\bibinfo{year}{2020}).
\bibitem[{Ng and Russell(2000)}]{ng2000algorithms}
\bibinfo{author}{A.~Y. Ng}, \bibinfo{author}{S.~J. Russell},
\newblock \bibinfo{title}{{Algorithms for inverse reinforcement learning}},
\newblock in: \bibinfo{booktitle}{International Conference on Machine
  Learning}, \bibinfo{year}{2000}, pp. \bibinfo{pages}{663--670}.
\bibitem[{Abbeel and Ng(2004)}]{abbeel2004apprenticeship}
\bibinfo{author}{P.~Abbeel}, \bibinfo{author}{A.~Y. Ng},
\newblock \bibinfo{title}{{Apprenticeship learning via inverse reinforcement
  learning}},
\newblock in: \bibinfo{booktitle}{21st International Conference on Machine
  Learning}, \bibinfo{year}{2004}, p.~\bibinfo{pages}{1}.
\bibitem[{Ziebart(2010)}]{ziebart2010modeling}
\bibinfo{author}{B.~D. Ziebart}, \bibinfo{title}{{Modeling purposeful adaptive
  behavior with the principle of maximum causal entropy}}, Ph.D. thesis,
  Carnegie Mellon University, \bibinfo{year}{2010}.
\bibitem[{Sadigh et~al.(2017)Sadigh, Dragan, Sastry, and
  Seshia}]{sadigh_active_2017}
\bibinfo{author}{D.~Sadigh}, \bibinfo{author}{A.~Dragan},
  \bibinfo{author}{S.~Sastry}, \bibinfo{author}{S.~Seshia},
\newblock \bibinfo{title}{{Active preference-based learning of reward
  functions}},
\newblock in: \bibinfo{booktitle}{Robotics: {Science} and {Systems} {XIII}},
  \bibinfo{year}{2017}, pp. \bibinfo{pages}{53--63}.
\bibitem[{Christiano et~al.(2017)Christiano, Leike, Brown, Martic, Legg, and
  Amodei}]{Christiano2017}
\bibinfo{author}{P.~F. Christiano}, \bibinfo{author}{J.~Leike},
  \bibinfo{author}{T.~B. Brown}, \bibinfo{author}{M.~Martic},
  \bibinfo{author}{S.~Legg}, \bibinfo{author}{D.~Amodei},
\newblock \bibinfo{title}{{Deep reinforcement learning from human
  preferences}},
\newblock \bibinfo{journal}{Neural Information Processing Systems}
  (\bibinfo{year}{2017}) \bibinfo{pages}{4300--4308}.
\bibitem[{Goyal et~al.(2019)Goyal, Niekum, and Mooney}]{goyal2019using}
\bibinfo{author}{P.~Goyal}, \bibinfo{author}{S.~Niekum}, \bibinfo{author}{R.~J.
  Mooney},
\newblock \bibinfo{title}{{Using natural language for reward shaping in
  reinforcement learning}},
\newblock \bibinfo{journal}{arXiv preprint arXiv:1903.02020}
  (\bibinfo{year}{2019}).
\bibitem[{Arumugam et~al.(2019)Arumugam, Lee, Saskin, and
  Littman}]{arumugam2019deep}
\bibinfo{author}{D.~Arumugam}, \bibinfo{author}{J.~K. Lee},
  \bibinfo{author}{S.~Saskin}, \bibinfo{author}{M.~L. Littman},
\newblock \bibinfo{title}{{Deep reinforcement learning from policy-dependent
  human feedback}},
\newblock \bibinfo{journal}{arXiv preprint arXiv:1902.04257}
  (\bibinfo{year}{2019}).
\bibitem[{Bajcsy et~al.(2017)Bajcsy, Losey, O’Malley, and
  Dragan}]{bajcsy2017learning}
\bibinfo{author}{A.~Bajcsy}, \bibinfo{author}{D.~P. Losey},
  \bibinfo{author}{M.~K. O’Malley}, \bibinfo{author}{A.~D. Dragan},
\newblock \bibinfo{title}{{Learning robot objectives from physical human
  interaction}},
\newblock \bibinfo{journal}{Machine Learning Research} \bibinfo{volume}{78}
  (\bibinfo{year}{2017}) \bibinfo{pages}{217--226}.
\bibitem[{Hadfield-Menell et~al.(2017)Hadfield-Menell, Milli, Abbeel, Russell,
  and Dragan}]{hadfield2017inverse}
\bibinfo{author}{D.~Hadfield-Menell}, \bibinfo{author}{S.~Milli},
  \bibinfo{author}{P.~Abbeel}, \bibinfo{author}{S.~J. Russell},
  \bibinfo{author}{A.~Dragan},
\newblock \bibinfo{title}{{Inverse reward design}},
\newblock in: \bibinfo{booktitle}{Neural Information Processing Systems},
  \bibinfo{year}{2017}, pp. \bibinfo{pages}{6765--6774}.
\bibitem[{Mindermann et~al.(2018)Mindermann, Shah, Gleave, and
  Hadfield-Menell}]{mindermann2018active}
\bibinfo{author}{S.~Mindermann}, \bibinfo{author}{R.~Shah},
  \bibinfo{author}{A.~Gleave}, \bibinfo{author}{D.~Hadfield-Menell},
\newblock \bibinfo{title}{{Active inverse reward design}},
\newblock \bibinfo{journal}{arXiv preprint arXiv:1809.03060}
  (\bibinfo{year}{2018}).
\bibitem[{Stiennon et~al.(2020)Stiennon, Ouyang, Wu, Ziegler, Lowe, Voss,
  Radford, Amodei, and Christiano}]{stiennon2020learning}
\bibinfo{author}{N.~Stiennon}, \bibinfo{author}{L.~Ouyang},
  \bibinfo{author}{J.~Wu}, \bibinfo{author}{D.~Ziegler},
  \bibinfo{author}{R.~Lowe}, \bibinfo{author}{C.~Voss},
  \bibinfo{author}{A.~Radford}, \bibinfo{author}{D.~Amodei},
  \bibinfo{author}{P.~F. Christiano},
\newblock \bibinfo{title}{{Learning to summarize with human feedback}},
\newblock \bibinfo{journal}{Neural Information Processing Systems}
  \bibinfo{volume}{33} (\bibinfo{year}{2020}) \bibinfo{pages}{3008--3021}.
\bibitem[{Ziegler et~al.(2019)Ziegler, Stiennon, Wu, Brown, Radford, Amodei,
  Christiano, and Irving}]{ziegler2019fine}
\bibinfo{author}{D.~M. Ziegler}, \bibinfo{author}{N.~Stiennon},
  \bibinfo{author}{J.~Wu}, \bibinfo{author}{T.~B. Brown},
  \bibinfo{author}{A.~Radford}, \bibinfo{author}{D.~Amodei},
  \bibinfo{author}{P.~Christiano}, \bibinfo{author}{G.~Irving},
\newblock \bibinfo{title}{{Fine-tuning language models from human
  preferences}},
\newblock \bibinfo{journal}{arXiv preprint arXiv:1909.08593}
  (\bibinfo{year}{2019}).
\bibitem[{Leike et~al.(2022)Leike, Schulman, and Wu}]{leike_schulman_wu_2022}
\bibinfo{author}{J.~Leike}, \bibinfo{author}{J.~Schulman},
  \bibinfo{author}{J.~Wu}, \bibinfo{title}{Our approach to alignment research},
  \bibinfo{year}{2022}. \URLprefix
  \url{https://openai.com/blog/our-approach-to-alignment-research/}.
\bibitem[{Skalse and Abate(2022)}]{Skalse2022-np}
\bibinfo{author}{J.~Skalse}, \bibinfo{author}{A.~Abate},
\newblock \bibinfo{title}{Misspecification in inverse reinforcement learning},
\newblock \bibinfo{journal}{arXiv preprint arXiv:2212.03201}
  (\bibinfo{year}{2022}).
\bibitem[{Milli and Dragan(2020)}]{milli2020literal}
\bibinfo{author}{S.~Milli}, \bibinfo{author}{A.~D. Dragan},
\newblock \bibinfo{title}{{Literal or pedagogic human? {A}nalyzing human model
  misspecification in objective learning}},
\newblock in: \bibinfo{booktitle}{Uncertainty in Artificial Intelligence},
  \bibinfo{year}{2020}, pp. \bibinfo{pages}{925--934}.
\bibitem[{Freedman et~al.(2021)Freedman, Shah, and Dragan}]{freedman2021choice}
\bibinfo{author}{R.~Freedman}, \bibinfo{author}{R.~Shah},
  \bibinfo{author}{A.~Dragan},
\newblock \bibinfo{title}{{Choice set misspecification in reward inference}},
\newblock \bibinfo{journal}{arXiv preprint arXiv:2101.07691}
  (\bibinfo{year}{2021}).
\bibitem[{Daniels-Koch and Freedman(2022)}]{daniels2022expertise}
\bibinfo{author}{O.~Daniels-Koch}, \bibinfo{author}{R.~Freedman},
\newblock \bibinfo{title}{The expertise problem: Learning from specialized
  feedback},
\newblock \bibinfo{journal}{arXiv preprint arXiv:2211.06519}
  (\bibinfo{year}{2022}).
\bibitem[{Bıyık and Sadigh(2018)}]{biyik_batch_2018}
\bibinfo{author}{E.~Bıyık}, \bibinfo{author}{D.~Sadigh},
\newblock \bibinfo{title}{{Batch active preference-based learning of reward
  functions}},
\newblock \bibinfo{journal}{arXiv preprint arXiv:1810.04303}
  (\bibinfo{year}{2018}).
\bibitem[{Evans et~al.(2016)Evans, Stuhlmüller, and
  Goodman}]{evans_learning_2016}
\bibinfo{author}{O.~Evans}, \bibinfo{author}{A.~Stuhlmüller},
  \bibinfo{author}{N.~D. Goodman},
\newblock \bibinfo{title}{{Learning the preferences of ignorant, inconsistent
  agents}},
\newblock in: \bibinfo{booktitle}{30th AAAI Conference on Artificial
  Intelligence}, \bibinfo{year}{2016}, pp. \bibinfo{pages}{323--329}.
\bibitem[{Bıyık et~al.(2020)Bıyık, Palan, Landolfi, Losey, and
  Sadigh}]{biyik_asking_2020}
\bibinfo{author}{E.~Bıyık}, \bibinfo{author}{M.~Palan},
  \bibinfo{author}{N.~C. Landolfi}, \bibinfo{author}{D.~P. Losey},
  \bibinfo{author}{D.~Sadigh},
\newblock \bibinfo{title}{{Asking easy questions: A user-friendly approach to
  active reward learning}},
\newblock in: \bibinfo{booktitle}{Conference on Robot Learning},
  \bibinfo{year}{2020}, pp. \bibinfo{pages}{1177--1190}.
\bibitem[{Bradley and Terry(1952)}]{bradley_rank_1952}
\bibinfo{author}{R.~A. Bradley}, \bibinfo{author}{M.~E. Terry},
\newblock \bibinfo{title}{{Rank analysis of incomplete block designs: I. The
  method of paired comparisons}},
\newblock \bibinfo{journal}{Biometrika} \bibinfo{volume}{39}
  (\bibinfo{year}{1952}) \bibinfo{pages}{324--345}.
\bibitem[{Liang et~al.(2022)Liang, Shu, Lee, and Abbeel}]{liang_reward_2022}
\bibinfo{author}{X.~Liang}, \bibinfo{author}{K.~Shu}, \bibinfo{author}{K.~Lee},
  \bibinfo{author}{P.~Abbeel},
\newblock \bibinfo{title}{{Reward uncertainty for exploration in
  preference-based reinforcement learning}},
\newblock \bibinfo{journal}{arXiv preprint arXiv:2205.12401}
  (\bibinfo{year}{2022}).
\bibitem[{Ramachandran and Amir(2007)}]{ramachandran_bayesian_2007}
\bibinfo{author}{D.~Ramachandran}, \bibinfo{author}{E.~Amir},
\newblock \bibinfo{title}{{Bayesian Inverse Reinforcement Learning.}},
\newblock in: \bibinfo{booktitle}{International Joint Conference on Artificial
  Intelligence}, volume~\bibinfo{volume}{7}, \bibinfo{year}{2007}, pp.
  \bibinfo{pages}{2586--2591}.
\bibitem[{Palan et~al.(2019)Palan, Shevchuk, Charles~Landolfi, and
  Sadigh}]{palan_learning_2019}
\bibinfo{author}{M.~Palan}, \bibinfo{author}{G.~Shevchuk},
  \bibinfo{author}{N.~Charles~Landolfi}, \bibinfo{author}{D.~Sadigh},
\newblock \bibinfo{title}{{Learning reward functions by integrating human
  demonstrations and preferences}},
\newblock in: \bibinfo{booktitle}{Robotics: {Science} and {Systems} {XV}},
  \bibinfo{year}{2019}, pp. \bibinfo{pages}{23--33}.
\bibitem[{Freedman et~al.(2020)Freedman, Borg, Sinnott-Armstrong, Dickerson,
  and Conitzer}]{freedman2020adapting}
\bibinfo{author}{R.~Freedman}, \bibinfo{author}{J.~S. Borg},
  \bibinfo{author}{W.~Sinnott-Armstrong}, \bibinfo{author}{J.~P. Dickerson},
  \bibinfo{author}{V.~Conitzer},
\newblock \bibinfo{title}{Adapting a kidney exchange algorithm to align with
  human values},
\newblock \bibinfo{journal}{Artificial Intelligence} \bibinfo{volume}{283}
  (\bibinfo{year}{2020}) \bibinfo{pages}{103261}.
\bibitem[{Gomez-Uribe and Hunt(2015)}]{gomez2015netflix}
\bibinfo{author}{C.~A. Gomez-Uribe}, \bibinfo{author}{N.~Hunt},
\newblock \bibinfo{title}{{The netflix recommender system: Algorithms, business
  value, and innovation}},
\newblock \bibinfo{journal}{ACM Transactions on Management Information Systems
  (TMIS)} \bibinfo{volume}{6} (\bibinfo{year}{2015}) \bibinfo{pages}{1--19}.
\bibitem[{Perano et~al.(2021)Perano, Casali, Liu, and
  Abbate}]{perano2021professional}
\bibinfo{author}{M.~Perano}, \bibinfo{author}{G.~L. Casali},
  \bibinfo{author}{Y.~Liu}, \bibinfo{author}{T.~Abbate},
\newblock \bibinfo{title}{{Professional reviews as service: A mix method
  approach to assess the value of recommender systems in the entertainment
  industry}},
\newblock \bibinfo{journal}{Technological Forecasting and Social Change}
  \bibinfo{volume}{169} (\bibinfo{year}{2021}) \bibinfo{pages}{120800}.
\bibitem[{Raza and Ding(2021)}]{raza2021news}
\bibinfo{author}{S.~Raza}, \bibinfo{author}{C.~Ding},
\newblock \bibinfo{title}{{News recommender system: A review of recent
  progress, challenges, and opportunities}},
\newblock \bibinfo{journal}{Artificial Intelligence Review}
  (\bibinfo{year}{2021}) \bibinfo{pages}{1--52}.
\bibitem[{Alamdari et~al.(2020)Alamdari, Navimipour, Hosseinzadeh, Safaei, and
  Darwesh}]{alamdari2020systematic}
\bibinfo{author}{P.~M. Alamdari}, \bibinfo{author}{N.~J. Navimipour},
  \bibinfo{author}{M.~Hosseinzadeh}, \bibinfo{author}{A.~A. Safaei},
  \bibinfo{author}{A.~Darwesh},
\newblock \bibinfo{title}{{A systematic study on the recommender systems in the
  E-commerce}},
\newblock \bibinfo{journal}{IEEE Access} \bibinfo{volume}{8}
  (\bibinfo{year}{2020}) \bibinfo{pages}{115694--115716}.
\bibitem[{Shah et~al.(2019)Shah, Gundotra, Abbeel, and
  Dragan}]{shah_feasibility_2019}
\bibinfo{author}{R.~Shah}, \bibinfo{author}{N.~Gundotra},
  \bibinfo{author}{P.~Abbeel}, \bibinfo{author}{A.~Dragan},
\newblock \bibinfo{title}{{On the feasibility of learning, rather than
  assuming, human biases for reward inference}},
\newblock in: \bibinfo{booktitle}{36th {International} {Conference} on
  {Machine} {Learning}}, \bibinfo{publisher}{PMLR}, \bibinfo{year}{2019}, pp.
  \bibinfo{pages}{5670--5679}.
\bibitem[{Chan et~al.(2021)Chan, Critch, and Dragan}]{chan_human_2021}
\bibinfo{author}{L.~Chan}, \bibinfo{author}{A.~Critch},
  \bibinfo{author}{A.~Dragan},
\newblock \bibinfo{title}{{Human irrationality: Both bad and good for reward
  inference}},
\newblock \bibinfo{journal}{arXiv preprint arXiv:2111.06956}
  (\bibinfo{year}{2021}).

\end{thebibliography}

\end{document}